\documentclass[12pt]{article}

\usepackage{theorem}
\usepackage{amssymb}

\newtheorem{definition}{Definition}[section]

\newtheorem{example}{Example}[section]

\newtheorem{proposition}{Proposition}[section]
\newenvironment{proof}{{\bf Proof:}}{\hfill $\square$\par}

\usepackage[a4paper]{geometry}    

\usepackage{url}
\usepackage{enumerate}
\usepackage{color}
\usepackage{algorithm} 
\usepackage{algorithmicx} 
\usepackage[noend]{algpseudocode}

\begin{document}
\title{\bf DeciLS-PBO: an Effective Local Search Method for Pseudo-Boolean Optimization}

\date{}

\author{\sffamily Luyu Jiang{$^{1,2}$}, Dantong Ouyang{$^{1,2}$}, Qi Zhang{$^{1,2}$}, Liming Zhang{$^{1,2,*}$}\\
{\sffamily\small $^1$ College of Computer Science and Technology, Jilin University, Changchun, 130012, China }\\
    {\sffamily\small $^2$ Key Laboratory of Symbolic Computation and Knowledge Engineering, Ministry of Education,}\\
    {\sffamily\small Jilin  University, Changchun, 130012, China }}
\renewcommand{\thefootnote}{\fnsymbol{footnote}}
\footnotetext[1]{Corresponding author. }

\maketitle

\begin{abstract}
Local search is an effective method for solving large-scale combinatorial optimization problems, and it has made remarkable progress in recent years through several subtle mechanisms. In this paper, we found two ways to improve the local search algorithms in solving Pseudo-Boolean Optimization(PBO): Firstly, some of those mechanisms such as unit propagation are merely used in solving MaxSAT before, which can be generalized to solve PBO as well; Secondly, the existing local search algorithms utilize the heuristic on variables, so-called score, to mainly guide the search. We attempt to gain more insights into the clause, as it plays the role of a middleman who builds a bridge between variables and the given formula. Hence, we first extended the combination of unit propagation-based decimation algorithm to PBO problem, giving a further generalized definition of unit clause for PBO problem, and apply it to the existing solver LS-PBO for constructing an initial assignment; then, we introduced a new heuristic on clauses, dubbed care, to set a higher priority for the clauses that are less satisfied in current iterations. Experiments on benchmarks from the most recent PB Competition, as well as three real-world application benchmarks including minimum-width confidence band, wireless sensor network optimization, and seating arrangement problems show that our algorithm DeciLS-PBO has a promising performance compared to the state-of-the-art algorithms.
\end{abstract}

\textbf{keywords}: pseudo-boolean optimization, local search, unit propagation

\section{Introduction}
Local search method is a rising star for solving combinatorial optimization problems in recent years, and the state-of-the-art local search-based incomplete MaxSAT solvers show promising performance even competitive to many complete solvers in recent MaxSAT Evaluations \cite{maxsat}. As a widely acknowledged efficient technique with a simple framework, local search works as an anytime solver and helps obtain a high-quality solution in any given time limit for solving large instances in many problem domains of combinatorial optimization \cite{ls1,ls2,ls3,ls4,ls5,ls6}. As the scale of the problem keeps increasing, complete methods based on the ideas including conflict-driven clause learning (CDCL) \cite{cdcl1,cdcl2,roundingsat} and Branch and Bound(BnB) \cite{ihs}, et al. suffer from serious time-consuming problems and are unable to obtain high-quality solutions in a short period of time. The local search algorithms, by contrast, show notable superiority in solving large-scale problems in a short time owing to their lightweight framework and subtle strategies. It is composed of an initial solution and an iterative process with a heuristic for finding better local solutions. Local search methods keep the iterative process until the termination condition (time limit in most cases) is reached, and then it returns the best solution obtained during the whole iterative process. Breakthroughs in local search-based solvers have taken place in the last decade since many effective new techniques have been proposed. Therefore, many incomplete solvers like DeciDist \cite{decidist}, CCEHC \cite{ccehc}, SATLike \cite{satlike}, etc., led to the success of modern local search-based MaxSAT solvers.

For problems such as SAT, MaxSAT, and Partial MaxSAT (PMS), formulas are encoded in conjunctive normal form (CNF). However, there exists a significant drawback which the express ability of CNF encoding is limited in handling problems with cardinality constraints. Indeed, there are many real-world situations such as scheduling, logic synthesis or verification, and discrete tomography, that involve cardinality constraints themselves. Based on the requirements above, many works seek to design an efficient encoding of cardinality constraints in CNF formulas\cite{hatt17,boud18,karp19}, and even a new encoding called Extended Conjunctive Normal Form (ECNF) \cite{ecnf} that can express cardinality constraints straightforward and need no auxiliary variables or clauses. Cardinality constraints are a special case of pseudo-boolean optimization (PBO). To further enlarge the scope of problems that can be expressed, nowadays studies start to focus on a more generalized problem, PBO. However, to the best of our knowledge, studies on PBO are very limited, furthermore, LS-PBO \cite{lspbo} is the first local search-based PBO solver in recent years.

In LS-PBO, Lei et al. \cite{lspbo} transformed the objective function of PBO into objective constraints, similar to soft clauses in weighted partial MaxSAT (WPMS), and treated PB constraints as hard constraints, similar to hard clauses in WPMS. Therefore, the modern techniques used in local search methods for solving (W)PMS can be applied to solving PBO as well. Therefore, after the transformation it allows one to use state-of-the-art local search-based PMS solvers right off-the-shelf as PBO solvers. Modern local search algorithms for solving PMS problems mainly focus on two components: a variable selection heuristic (i.e., a scoring function for variables) and a clause weighting scheme. They are elaborately designed to help the search avoid getting stuck at the local optima and many algorithms \cite{nudist,ccehc,satlike} have proposed different ways to help obtain better solutions. However, algorithms in recent years focus too much on variables, and the information about clauses, which can directly affect the satisfiability of the given formula, is ignored. As a connection between variables and the formula, the assignments to variables determine the true literals directly, and then the clauses containing true literals are determined as satisfied, finally, the formula is satisfied if all clauses are satisfied. Therefore, we realize it is important to gain more insights into clauses and propose a new heuristic on clause called ”care”. “Score” indicates the benefits of flipping a variable, whereas “care” reflects the total number of times a clause is falsified in the current stage.

Besides the variable selection heuristic and the clause weighting scheme, modern local search solvers also start to apply unit propagation(UP) for improving their initial process. To the best of our knowledge, DeciLS \cite{decidist} is the first contribution that combines UP-based decimation algorithm with local search for PMS. Recently, many studies have successfully applied UP-based decimation algorithms to generate better initial assignments than random ones. In \cite{satlike3}, Cai et al. applied the UP-based decimation algorithm to improve the initial assignments for local search algorithm in solving PMS problems. In 2021, Lei et al. \cite{ecnf} introduced the definition of generalized unit clause and further integrated UP to design a decimation algorithm for extended PMS with cardinality constraints, to produce initial assignments for LS-ECNF. Alongside the previous studies, we are motivated by integrating UP into the initial stage of local search method in solving PBO problems. Therefore, we proposed new definitions of improved generalized unit clauses to make UP applicable to PBO.

Based on the analyses above, this paper improved the local search PBO solver by presenting two new ideas. The contributions can be  summarized as follows:
\begin{itemize}
    \item First, we break the original mode that local search algorithm guided by a heuristic on variables called score, and propose a new heuristic on clauses, dubbed care, to guide the search straightforwardly to focus on the clauses (constraints) hardly to be satisfied in the current situation. We propose a scheme called Care-FC scheme based on the heuristic about care value, and apply it to the situation when the algorithm gets stuck at local optima.
    \item Second, we extend the use of unit propagation-based decimation algorithms to PBO problem for generating an improved assignment in the initial stage. In order to make it applicable, we also propose new definitions of improved generalized unit clauses and make some modifications to the original UP-based decimation. Thus, instead of generating an initial assignment randomly, we utilize unit propagation to improve the quality of the initial assignment.
    \item Finally, We carry out experiments to compare DeciLS-PBO against the state-of-the-art local search solver LS-PBO \cite{lspbo} and a newly published implicit hitting set-based solver PBO-IHS \cite{ihs} on benchmarks from three real-world application instances including minimum-width confidence band(MWCB), wireless sensor network optimization(WSNO), and seating arrangement problems(SAP), as well as the benchmarks from the most recent PB Competition. Experimental results show that our DeciLS-PBO outperforms the competing solvers on most of the instances mentioned above.
\end{itemize}

The remainder of this paper is organized as follows. Section \ref{sec:pre} introduces the preliminary knowledge. Section \ref{idea} presents the two new ideas including the heuristic on falsified clauses referred to as Care-FC scheme, and the improved UP-based decimation algorithm for PBO. Section \ref{pbo} presents our new algorithm DeciLS-PBO, which combined the two new ideas into the existing algorithm LS-PBO. Section \ref{sec:exp} provides the experimental results. And we draw the conclusion in Section \ref{sec:con}.

\section{Preliminaries}
\label{sec:pre}
Given a set of $n$ Boolean variables $\left\{x_{1}, x_{2}, \ldots, x_{n}\right\}$, a literal $l_{i}$ is either a variable $x_{i}$ or its negation $\neg x_{i}$.  A clause ci is a disjunction of literals. A conjunctive normal form (CNF) formula F is a conjunction of clauses. Each variable has its domain $\left\{0,1\right\}$, and an assignment $\alpha$ is a mapping from the variables to their values (0 or 1).  A clause is satisfied if it has at least one true literal, and falsified if all the literals in the clause are false literals. The goal of satisfiability problem is to judge whether a given CNF formula is satisfiable or not, and that of its improved version, MaxSAT problem, is to find a complete assignment that maximizes the number of satisfied clauses of the given CNF formula. Partial MaxSAT is an important variant of MaxSAT, and it divides the clauses into hard clauses and soft clauses. A hard clause is defined as a clause that must be satisfied in a feasible assignment, whereas a soft clause is defined to be a clause irrelevant to the feasibility of the solution but affects the quality of the solution. The goal of PMS is to find a complete assignment that satisfies all hard clauses and minimizes the number of falsified soft clauses. 

A pseudo-Boolean (PB) constraint is an inequality of the form $ \sum_{i} a_{i}l_{i} \geq B$, where $a_{i},B \in N^+ , l_{i}\in \left\{x_{i},\neg x_{i}\right\}$, and it can be equivalent to a 0-1 integer linear inequality by replacing literal $\neg x_{i}$ with $1-x_{i}$. Each literal $l_{i}$ is associated with a positive integer coefficient $a_{i}$, and B denotes the bound of the PB constraint. A PB constraint is satisfied if the 0-1 integer linear inequality holds. A PB constraint is simplified as a cardinality constraint when all coefficients are set to 1. The Pseudo-Boolean Optimization problem consists of two components over the Boolean variables $\left\{x_{1},x_{2},\ldots,x_{n}\right\}$: one is a set of PB constraints, the other is an objective function that minimizes $ \sum_{i}^{} b_{i}l_{i}, b_{i}\in N^+$. The goal of PBO is to minimize the value of the objective function under the premise that all PB constraints must be satisfied.

A PB formula is a conjunction of PB constraints, and it reduces to a formula expressed in extended conjunctive normal form (ECNF) when all coefficients of literals are set to 1. Moreover, for formulas expressed in conjunctive normal form (CNF), each literal in the form $l_{1}\wedge \ldots \wedge l_{j}$ can be expressed equally as a constraint $l_{1}+\ldots+l_{j}\geq1$. Therefore, formulas in both CNF and ECNF formats are special cases of PB formulas. In CNF encoding, a clause is satisfied if it contains at least one true literal. In ECNF encoding, a clause with cardinality constraints is satisfied if the number of true literals it contains reaches the given cardinality. For a pseudo-Boolean optimization problem, a PB constraint is satisfied if the sum of the coefficient of true literals reaches the bound.

\section{Main ideas}
\label{idea}

In this section, we present two new ideas in our algorithm.
\subsection{Heuristic with emphasis on falsified clauses}

\renewcommand{\floatpagefraction}{.9}
\begin{algorithm}[htbp] 
    \renewcommand{\algorithmicrequire}{\textbf{Input:}}
	\renewcommand{\algorithmicensure}{\textbf{Output:}}
	\caption{Modern local search algorithm framework for PMS/EPMS/PBO}
	\label{alg:1}
	\begin{algorithmic}[1]
            \Require 
		  PMS/EPMS/PBO instance $F$, cutoff time $cutoff$
            \Ensure 
             A feasible assignment $\alpha$ and its cost (objective value) for PMS and EPMS (PBO)
		\State $\alpha \gets $ an initial complete assignment; $\alpha* \gets \emptyset$;
                \While{elapsed time \textless $cutoff$}
			\If{$\alpha$ is feasible} \& $cost(\alpha)\textless cost$* 
                    		\State $\alpha* \gets \alpha; cost* \gets cost(\alpha)$;
			\EndIf
                    \If {${\exists}$ variable $x$ such that $score(x)\textgreater 0$}
                       \State $v\gets$variable with the highest $score(v)$;
                    \Else
                       \State update clause weight;
                       \If{${\exists}$falsified hard clauses } 
				\State $c \gets$ a random falsified hard clause;
			  \Else
				\State $c \gets$ a random falsified soft clause;
		         \EndIf
			  \State $v\gets$variable with highest $score(v)$ in $c$;
                    \EndIf
			\State $\alpha \gets \alpha$ with $v$ flipped;
                \EndWhile
	\State \Return{$(\alpha*,cost*)$ if $\alpha*$ is feasible, otherwise "no feasible solution"}
	\end{algorithmic} 
\end{algorithm}

A general framework of modern local search algorithms is presented in Algorithm \ref{alg:1}. It keeps on the loop and chooses a variable to flip in each turn. The algorithm is mainly guided by the heuristic on variables, i.e., the measurement called score. Score(v) is an indicator used to measure whether flipping the assignment of variable v can bring positive benefits, i.e., making more falsified clauses to be satisfied. score(v)\textgreater0 means that flipping the assignment of variable v can bring positive benefits, and the value of score(v) indicates how many positive benefits are gained by flipping the assignment of variable v.

The marrow of these algorithms is two components:  \textsl{a variable selection heuristic} and  \textsl{a clause weighting scheme}. The variable selection heuristic helps choose the suitable variable to flip in the current state, whereas the clause weighting scheme takes effect when the algorithm gets “stuck”, which means there is no variable with positive scores exists thus the variable selection heuristic cannot continue guiding the search. More detailly, the clause weighting scheme works by updating the weight of clauses, resulting in the update of variables’ scores. Therefore, both the variable selection heuristic and the clause weighting scheme emphasize the score of variables. Nevertheless, the studies about heuristics on clauses is ignored in recent years.
\begin{algorithm}[h] 
    \renewcommand{\algorithmicrequire}{\textbf{Input:}}
	\renewcommand{\algorithmicensure}{\textbf{Output:}}
	\caption{Modern local search algorithm framework with Care-FC scheme}
   \label{alg:2}
	\begin{algorithmic}[1]
            \Require PMS/EPMS/PBO instance $F$, cutoff time $cutoff$
            \Ensure A feasible assignment $\alpha$ and its cost (objective) for PMS and EPMS (PBO)
		\State $\alpha \gets $ an initial complete assignment; $\alpha* \gets \emptyset$;
                \While{elapsed time \textless $cutoff$}
			\If{$\alpha$ is feasible} \& $cost(\alpha)\textless cost$* 
                    		\State $\alpha* \gets \alpha; cost* \gets cost(\alpha)$;
			\EndIf
                    \If {${\exists}$ variable $x$ such that $score(x)\textgreater 0$}
                       \State $v\gets$variable $v$ with highest $score(v)$;
                    \Else
                       \State update clause weights;
                       \If{${\exists}$falsified hard clauses } 
				\State with probability $p$: $c \gets$ a random falsified hard clause;
				\State with probability $1-p$: $c \gets$ a hard clause with highest $care(c)$;
			  \Else
				\State $c \gets$ a random falsified soft clause;
		         \EndIf
			  \State $v\gets$variable with highest $score(v)$ in $c$;
                    \EndIf
			\State $\alpha \gets \alpha$ with $v$ flipped;
                \EndWhile
	\State \Return{$(\alpha*,cost*)$ if $\alpha*$ is feasible, otherwise "no feasible solution"}
	\end{algorithmic} 
\end{algorithm}

\begin{definition}
 (care): The care of a PB constraint $c$, denoted by $care(c)$, is the total falsified count of constraint $c$.
\end{definition}

A given formula consists of all the clauses, whereas a clause consists of several variables or their negations. Clauses play the role of a middleman who builds a bridge between literals and the given formula. Therefore, gaining more insights into clauses also matters. In this paper, we propose a new concept, called care, to measure the total falsified number of a clause. We lean to set a higher priority for the clauses with higher care values and then choose a suitable variable from clauses with the highest care value to flip when the algorithm gets stuck. Algorithm \ref{alg:2} shows the general framework with our new heuristic, i.e., Care-FC scheme. We adopt a probability $p$ to switch between randomly selecting a falsified clause and selecting the clause with the highest care value. Instead of choosing a falsified hard clause randomly, we added the possibility of choosing the hard clause with the highest care value, which indicates the chosen clause has the least satisfied count during the search, thus it needs more "care" than other clauses. Compared to the original mode of LS-PBO, the Care-FC scheme assigns higher precedence to the hard clauses that are hard to be satisfied. For a hard clause with the highest care value, the current successive complete assignments are unfavorable to satisfy it. In order to find a feasible solution at the earliest opportunity, we prefer to flip a variable from the falsified hard clause with the highest care value directly rather than flip from a random falsified hard clause. Therefore, Care-FC scheme provides a probability for the algorithm to obtain a feasible solution as soon as possible, thus more time is saved for the algorithm to improve the solution.

\subsection{Combining a unit propagation-based decimation algorithm}
Unit clause is a special type of clause with a single literal in it. Unit propagation has been used in solving SAT and MaxSAT. It works by collecting all unit clauses, and then assigning their literals to be true to satisfy all the unit clauses. The assignment obtained by unit propagation guarantees at least the unit clauses can be satisfied, and another advantage is that the given formula is simplified as well. Compared to a random complete assignment, it provides an intuitive way to improve the quality of the initial solution. LS-ECNF\cite{ecnf} considers the cardinality constraints and extends its definition to generalized unit clause. In this paper, we further generalize this definition to pseudo-Boolean constraints and introduce two kinds of improved generalized unit clauses.
\begin{definition}
 (1-of-all generalized unit clause): for a PB constraint $\sum_{i} a_{i}l_{i} \geq B$, the literal $l_{j}$ with the highest coefficient $a_{j}$ is a 1-of-all generalized unit clause, if $(\sum_{i} a_{i})-a_j\le B$.
\end{definition}
\begin{definition}
(generalized unit constraint): a PB constraint $\sum_{i} a_{i}l_{i} \geq B$ is a generalized unit constraint if and only if $\sum_{i} a_{i}\leq B$.
\end{definition}
\begin{definition}
(all-of-all generalized unit clause): every literal $l_{i}$ in a generalized unit constraint is an all-of-all generalized unit clause.
\end{definition}
\begin{example}
Consider a hard PBO constraint $5x_{1}+x_{2}+x_{3}+x_{4}\geq6$. To make it satisfiable, we must assign $x_{1}=1$ since there are two reasons: 1) $x_{1}$ has the highest coefficient, thus it has the greatest influence on the satisfiability of the constraint. 2) even if all the other literals are made true except $x_{1}$, the constraint will not be satisfied because 3\textless6, thus the assignments on these variables take no effect on the satisfiability of the constraint. Hence, $x_{1}=1$ is necessary for this constraint to be satisfied. And in this case, $x_{1}$ is a 1-of-all generalized unit clause; Consider another hard PBO constraint $2x_{1}+x_{2}+x_{3}+x_{4}\geq5$, we must make true every literal, i.e., $(x_{1}=1, x_{2}=1, x_{3}=1, x_{4}=1)$, since only in that case can the constraint be satisfied. Besides, the relationship between 1-of-all generalized unit clauses and all-of-all generalized unit clauses is described in Proposition \ref{pro}. Conversely, given a 1-of-all generalized unit clause, it will be an all-of-all generalized unit clause only when its variable exists in a generalized unit constraint.
\end{example}

\begin{proposition}
\label{pro}
Given a generalized unit PB constraint c, the literal with the highest coefficient in it must be a 1-of-all generalized unit clause.
\end{proposition}

\begin{proof}
Assume that the given PB constraint $c$ forms like $\sum_{i} a_{i}l_{i} \geq B$, and the highest coefficient in $c$ is $a_{j}$. Since $c$ is a generalized unit constraint, there must be $\sum_{i} a_{i}\leq B$. Hence, there must be $(\sum_{i} a_{i})-a_j\le B$ as each $a_{i}$ is a positive integer. As the definition of 1-of-all generalized unit clause describes, the unit clause relates literal $l_{j}$ must be a 1-of-all generalized unit clause. 
\end{proof}

\begin{algorithm*}[] 
    \renewcommand{\algorithmicrequire}{\textbf{Input:}}
	\renewcommand{\algorithmicensure}{\textbf{Output:}}
	\caption{IGUP-Decimation}
	\label{alg:3}
	\begin{algorithmic}[1]
            \Require PBO instance $F$
            \Ensure An assignment of variables in $F$
                \While{${\exists}$ unassigned variables}
			\If{${\exists}$ improved generalized hard unit clauses }
                    		\State pick an improved generalized hard unit clause $c$ randomly;
 				\State perform improved generalized unit propagation by $c$; 
                    \ElsIf {${\exists}$ improved generalized} soft unit clauses 
                       \State pick an improved generalized soft unit clause $c$ randomly;
			  \State perform improved generalized unit propagation by $c$; 
                    \Else
                       \State $x \gets $pick an unassigned variable randomly;
                       \State assign $x$ with a random value (0 or 1), simplify $F$ accordingly;
                    \EndIf
                \EndWhile
	\State \Return{the assignment to variables of $F$}
	\end{algorithmic} 
\end{algorithm*}

Improved generalized unit propagation: For a given PBO instance, we collect all its 1-of-all generalized unit clauses and all-of-all generalized unit clauses, and set the literals in these improved generalized unit clauses to be true, therefore, the PBO instance is simplified accordingly.
The IGUP-based decimation algorithm is shown in Algorithm \ref{alg:3}. The algorithm works by iteratively assigning variables one by one until all variables have already been assigned. In each step, there are three alternative cases. (1) Firstly, if there exist improved generalized hard unit clauses, the algorithm picks one randomly. If the chosen clause has a contradictory hard unit clause, the variable x that relates the two contradictory clauses is assigned randomly; otherwise, it continues to process unit propagation. (2) Secondly, if there only exist improved generalized soft unit clauses, the algorithm picks one randomly and works a similar way as the case above. (3) Lastly, if no unit clause exists, the algorithm picks an unassigned variable, assigns it randomly and the formula is simplified accordingly.

\section{A new local search algorithm for PBO}
\label{pbo}
In this section, we give a detailed description of the new algorithm DeciLS-PBO.  Based on the work of LS-PBO, we added an UP-based decimation algorithm for generating initial assignments and a new heuristic emphasizing clauses called care as well. And both new ideas are introduced in the previous section. 

The outline of DeciLS-PBO is described in Algorithm \ref{alg:4}. In the initial stage, the algorithm calls the UP-based decimation to generate an initial assignment (line 1), then the algorithm goes into the main loop and terminates when the running time reaches the time limit (lines 2$\sim$15). $\alpha*$ denotes the best feasible assignment and $cost*$ denotes the cost value of $\alpha*$.

In the main loop, the algorithm first checks if the assignment in the current iteration is the best feasible assignment, if so, $\alpha*$ and $cost*$ are updated respectively (lines 3$\sim$4). Then, the algorithm checks if there exist variables with positive scores, if so, the algorithm picks the variable with the highest score to flip (lines 5$\sim$6, and line 15); Otherwise, the algorithm gets stuck at a local optimum and applies the clause weighting scheme proposed in LS-PBO to update the weights of constraints (line 8). Then, in the stage of random perturbation (lines 9$\sim$14), if there exist falsified hard clauses, the algorithm chooses a random falsified hard clause with probability $p$ and chooses the falsified hard clause with the highest care value with probability $1-p$, instead of providing only one way as LS-PBO which is choosing a random falsified hard clause. Then, the algorithm picks the variable from the chosen clause with the highest score to flip (lines 14$\sim$15). The main loop keeps executing until the time limit is reached.

\begin{algorithm*}[] 
    \renewcommand{\algorithmicrequire}{\textbf{Input:}}
	\renewcommand{\algorithmicensure}{\textbf{Output:}}
	\caption{DeciLS-PBO}
	\label{alg:4}
	\begin{algorithmic}[1]
            \Require A PBO instance $F$, cutoff time $cutoff$
            \Ensure A feasible solution $\alpha$ and its objective value
		\State $\alpha \gets $ IGUP-Decimation; $\alpha* \gets \emptyset$;
                \While{elapsed time \textless $cutoff$}
			\If{${\not\exists}$falsified hard clauses \& $cost(\alpha)\textless cost$* }
                    		\State $\alpha* \gets \alpha; cost* \gets cost(\alpha)$;
			\EndIf
                    \If {${\exists}$ variable $x$ such that $score(x)\textgreater 0$}
                       \State $v \gets$ variable $v$ with the highest $score(v)$;
                    \Else
                       \State update weights of clauses by Weighting-PBO scheme;
                       \If{${\exists}$falsified hard clauses } 
				\State with probability $p$: $c \gets$ a random falsified hard clause;
				\State with probability $1-p$: $c \gets$ a hard clause with highest $care(c)$;
			  \Else
				\State $c \gets$ a random falsified soft clause;
		         \EndIf
			  \State $v \gets$ the variable with the highest $score(v)$ in $c$;
                    \EndIf
			\State $\alpha \gets \alpha$ with $v$ flipped;
                \EndWhile
	\State \Return{$(\alpha*,cost*)$ if $\alpha*$ is feasible, otherwise "no feasible solution"}
	\end{algorithmic} 
\end{algorithm*}

Finally, the algorithm returns the best found assignment alpha and its cost if alpha is feasible, otherwise, the algorithm fails to find a feasible solution and returns “no feasible solution”.

\section{Experimental results}
\label{sec:exp}

We evaluate DeciLS-PBO on PBO instances from three real-world application problems, including minimum width confidence band (MWCB), wireless sensor network optimization (WSNO), and seating arrangement problems (SAP). A detailed description of these three problems can be found in literature \cite{lspbo}. The benchmarks above and the source code of LS-PBO were kindly provided by Prof. Cai at \url{https://lcs.ios.ac.cn/~caisw/Resource/LS-PBO/}. In addition, we also evaluated our solver on the OPT-SMALL-INT benchmark from the most recent Pseudo-Boolean Competition 2016 \cite{pb16}. Table \ref{tab:infor} shows the range of variables and constraints for each type of instance. Literature \cite{lspbo} confirmed that LS-PBO outperforms the state-of-the-art solvers LS-ECNF \cite{ecnf}, Loandra \cite{loandra}, RoundingSAT \cite{roundingsat}, Open-WBO \cite{openwbo} and SATLike-c \cite{satlike} on the benchmarks above. Hence, in this paper, we compare our DeciLS-PBO with LS-PBO and newly proposed state-of-the-art PBO solver PBO-IHS \cite{ihs}.

\begin{table}[H]
\caption{The range of the number of variables and constraints for each kind of instance}
\label{tab:infor}
\begin{tabular}{llll}
\hline
Instances  & \#inst &  \#variable          & \#constraint         \\ \hline
MWCB      & 24 & 974000$\sim$11381750 & 591827$\sim$10522608 \\
WSNO      & 18 & 58580$\sim$2311113   & 214679$\sim$26701722 \\
SAP       & 21 & 1650$\sim$11550      & 7756$\sim$108474     \\
PB16      & 1600 & 2$\sim$400896        & 4$\sim$521620       \\ \hline
\end{tabular}
\end{table}

We implemented our solver DeciLS-PBO based on the code of LS-PBO in C++ and complied it using g++ with -O3 option. The experiments were conducted on a server using 2.60GHz Intel(R) Xeon(R) E5-2690 CPU, 60GB RAM, running the Ubuntu 18.04 Linux operating system. We set two different time limits: 300s and 3600s. And for the solvers using stochastic method, we conducted 30 runs per instance with various seeds and reported the minimum, median, and maximum values over these 30 runs. For parameters in DeciLS-PBO, we keep the settings for those inherited from LS-PBO, and for the only parameter $p$ we introduced in Care-FC scheme, the parameter domain of $p$ is [0, 1], and we set the default value of $p$ to 0.5.

\subsection{Experimental results on MWCB}

Experimental results on MWCB instances are shown in Table \ref{tab:mwcb300} and Table \ref{tab:mwcb3600}. For each instance, the best results among the three algorithms are marked in bold font. PBO-IHS performs the worst on both 300s and 3600s time limit. Hence, we focus on the comparison between solver LS-PBO and our solver DeciLS-PBO. 

Since 300s time limit experiment shows the performance of an algorithm in the early period, the schemes used in the initial stage have a greater impact on the performance. In the initial stage, LS-PBO generates an initial assignment by assigning all variables to 0, whereas our DeciLS-PBO calls the UP-based decimation to generate an initial assignment. For 300s time limit, Table \ref{tab:mwcb300} shows the best objective value obtained by DeciLS-PBO dominates that of LS-PBO on most of the instances, which indicates our unit propagation-based initial assignment is effective. 

For 3600s time limit, Table \ref{tab:mwcb3600} shows that both LS-PBO and DeciLS-PBO achieve better results than those of 300s time limit, and the best objective value obtained by DeciLS-PBO dominates that of LS-PBO on all instances except 1800\_200\_90 and 1800\_250\_90 instances. Therefore, Table \ref{tab:mwcb300} and Table \ref{tab:mwcb3600} demonstrate that our DeciLS-PBO has remarkable performance in the early stage and later as well.

\begin{table}[]
\caption{Results on MWCB under 300s time limit}
\label{tab:mwcb300}

\begin{tabular}{l|lll}
\hline
Instances     & Deci-LSPBO                       & LS-PBO                     & PBO-IHS \\
\hline
n m k         & Min{[}median,max{]}              & min{[}median,max{]}        &         \\ 
\hline
1000\_200\_90 & \textbf{109977}{[}+1172,+2425{]} & 110284{[}+1893,+3339{]}    & 190020  \\
1000\_250\_90 & \textbf{147594}{[}+1989.5,+3385{]}        & 148867{[}+894.5,+2729{]}   & 239610  \\
1200\_200\_90 & \textbf{111719}{[}+1804,+3109{]}          & 111821{[}+1823.5,+3020{]}  & 234000  \\
1200\_250\_90 & \textbf{150846}{[}+2572,+3543{]}          & 151742{[}+1777.5,+3645{]}  & 305445  \\
1400\_200\_90 & \textbf{110826}{[}+2096,+3867{]}          & 111916{[}+1453,+3432{]}    & 234000  \\
1400\_250\_90 & \textbf{150754}{[}+2747,+4344{]}          & 152023{[}+1593,+4562{]}    & 305445  \\
1600\_200\_90 & \textbf{138286}{[}+6476.5,+31316{]}       & 138443{[}+6465,+23994{]}   & 367252  \\
1600\_250\_90 & \textbf{186158}{[}+12162,+20707{]}        & 188021{[}+10142,+27216{]}  & 463177  \\
1800\_200\_90 & \textbf{222512}{[}+3393.5,+9378{]}        & 224517{[}+4030.5,+10987{]} & 386152  \\
1800\_250\_90 & \textbf{280395}{[}+4069,+10981{]}         & 283480{[}+5948,+12836{]}   & 482092  \\
2000\_200\_90 & \textbf{246473}{[}+5329,+11841{]}         & 248436{[}+5718,+11673{]}   & 400822  \\
2000\_250\_90 & \textbf{312680}{[}+5366,+11589{]}         & 318357{[}+3778.5,+10009{]} & 496891  \\
1000\_200\_95 & \textbf{117003}{[}+1114,+2559{]}          & 117427{[}+772.5,+2354{]}   & 190020  \\
1000\_250\_95 & \textbf{156678}{[}+2009,+2835{]}          & 157299{[}+1177.5,+3445{]}  & 239160  \\
1200\_200\_95 & \textbf{118414}{[}+1281,+2392{]}          & 118917{[}+607.5,+2204{]}   & 234000  \\
1200\_250\_95 & 159985{[}+1616.5,+3062{]}                 & \textbf{159966}{[}+1463,+2528{]}    & 305445  \\
1400\_200\_95 & \textbf{118617}{[}+1316,+2247{]}          & 118783{[}+953,+1889{]}     & 234000  \\
1400\_250\_95 & \textbf{161486}{[}+1378,+2599{]}          & 161752{[}+1275.5,+2534{]}  & 305445  \\
1600\_200\_95 & \textbf{184641}{[}+7601,+18174{]}         & 186543{[}+5680,+13354{]}   & 367252  \\
1600\_250\_95 & \textbf{237873}{[}+9365,+16938{]}         & 238974{[}+6825,+21612{]}   & 463177  \\
1800\_200\_95 & \textbf{251765}{[}+4326,+8574{]}          & 254319{[}+2541,+6949{]}    & 386152  \\
1800\_250\_95 & \textbf{315004}{[}+4428,+9553{]}          & 317284{[}+3512,+8087{]}    & 482092  \\
2000\_200\_95 & \textbf{273807}{[}+6588,+9950{]}          & 275854{[}+4289.5,+8439{]}  & 400822  \\
2000\_250\_95 & 343782{[}+5361,+9862{]}                   & \textbf{342949}{[}+5658.5,+10928{]} & 496891  \\ 
\hline
\end{tabular}

\end{table}

\begin{table}[]
\caption{Results on MWCB under 3600s time limit}
\label{tab:mwcb3600}

\begin{tabular}{l|lll}
\hline
Instances     & Deci-LSPBO                       & LS-PBO                      & PBO-IHS       \\ \hline
n m k         & min{[}median,max{]}              & min{[}median,max{]}         &        \\ \hline
1000\_200\_90 & \textbf{109795}{[}+1448,+2595{]} & 110316{[}+1009,+2757{]}     & 190020 \\
1000\_250\_90 & \textbf{147075}{[}+2687,+3922{]}          & 148350{[}+1690,+2795{]}     & 239610 \\
1200\_200\_90 & \textbf{110396}{[}+2153,+3695{]}          & 112550{[}+612.5,+2755{]}    & 234000 \\
1200\_250\_90 & \textbf{149924}{[}+2673,+3936{]}          & 151045{[}+1846,+3393{]}     & 305445 \\
1400\_200\_90 & \textbf{110085}{[}+2064,+3238{]}          & 111154{[}+1414,+3372{]}     & 234000 \\
1400\_250\_90 & \textbf{150570}{[}+1305,+3137{]}          & 151284{[}+1858,+3831{]}     & 305445 \\
1600\_200\_90 & \textbf{135717}{[}+8925,+21157{]}         & 136586{[}+5815,+31150{]}    & 367252 \\
1600\_250\_90 & \textbf{186249}{[}+13710,+28001{]}        & 188107{[}+16757.5,+28357{]} & 463177 \\
1800\_200\_90 & 220705{[}+5386,+10328{]}         & \textbf{217516}{[}+7570,+12295{]}    & 386152 \\
1800\_250\_90 & 279648{[}+2201,+8623{]}          & \textbf{278953}{[}+3702,+8000{]}     & 482092 \\
2000\_200\_90 & \textbf{244478}{[}+4211,+9719{]}          & 248243{[}+3687,+7869{]}     & 400822 \\
2000\_250\_90 & \textbf{307565}{[}+5381.5,+10326{]}       & 310396{[}+3932,+8922{]}     & 496891 \\
1000\_200\_95 & \textbf{116537}{[}+1148,+1951{]}          & 117111{[}+810,+2039{]}      & 190020 \\
1000\_250\_95 & \textbf{156144}{[}+1314,+2396{]}          & 157078{[}+805,+2548{]}      & 239160 \\
1200\_200\_95 & \textbf{118361}{[}+630,+2059{]}           & 118458{[}+938,+1653{]}      & 234000 \\
1200\_250\_95 & \textbf{159357}{[}+1469,+2741{]}          & 160060{[}+1219,+2785{]}     & 305445 \\
1400\_200\_95 & \textbf{118607}{[}+858,+2194{]}           & 118712{[}+1043,+2153{]}     & 234000 \\
1400\_250\_95 & \textbf{161040}{[}+1420,+2828{]}          & 161259{[}+1298,+2627{]}     & 305445 \\
1600\_200\_95 & \textbf{185265}{[}+6501,+12921{]}         & 188123{[}+4531,+14320{]}    & 367252 \\
1600\_250\_95 & \textbf{237184}{[}+6115,+14228{]}         & 238275{[}+5142,+14419{]}    & 463177 \\
1800\_200\_95 & \textbf{253035}{[}+2776,+4927{]}          & 255266{[}+1559,+3244{]}     & 386152 \\
1800\_250\_95 & \textbf{314056}{[}+3913,+8448{]}          & 315177{[}+3424,+9343{]}     & 482092 \\
2000\_200\_95 & \textbf{273052}{[}+5357,+7553{]}          & 274852{[}+3834,+6921{]}     & 400822 \\
2000\_250\_95 & \textbf{340105}{[}+7098,+9380{]}          & 341764{[}+5559.5,+8695{]}   & 496891 \\ \hline
\end{tabular}

\end{table}

\subsection{Experimental results on WSNO}

Experimental results on WSNO instances are shown in Table \ref{tab:wsno300} and Table \ref{tab:wsno3600}.  For each instance, the best results among the three algorithms are marked in bold font. The instances that an algorithm fails to find a feasible solution are marked "N/A". The results that are proven to be optimal by PBO-IHS are marked "*". Other than the two smallest instances 100\_40\_4 and 100\_40\_6, the performance of PBO-IHS lags far behind the other two algorithms on both 300s and 3600s time limit.

For 300s time limit, DeciLS-PBO obtains the same best results as LS-PBO on smaller instances but fails to find a feasible solution on larger instances. The time spent on unit propagation is related to the size of the instance. For DeciLS-PBO, the unit propagation process in the initial stage brings extra time consumption, which results in bad performance on larger instances.

For 3600s time limit, both DeciLS-PBO and LS-PBO obtain the same best results on all instances. However, the worst results obtained among 30 runs of DeciLS-PBO surpass that of LS-PBO on most instances, which indicates 2 observations: First, unit propagation-based decimation brings extra time cost but it also improves the initial assignment; Second, our DeciLS-PBO has a steadier search ability than LS-PBO, as its results among 30 runs do not vary a lot.

\begin{table}[]
\caption{Results on WSNO under 300s time limit}
\label{tab:wsno300}

\begin{tabular}{llll}
\hline
\multicolumn{1}{l|}{Instances}   & Deci-LSPBO              & LS-PBO                 & PBO-IHS \\ \hline
\multicolumn{1}{l|}{n m k}       & min{[}median,max{]}     & min{[}median,max{]}    &         \\ \hline
\multicolumn{1}{l|}{100\_40\_4}  & \textbf{210}{[}\textbf{+0},+3{]} & \textbf{210}{[}\textbf{+0},+12{]}        & 870     \\
\multicolumn{1}{l|}{150\_60\_4}  & \textbf{602}{[}\textbf{+0},\textbf{+0}{]}          & \textbf{602}{[}\textbf{+0},+2{]}         & 1180    \\
\multicolumn{1}{l|}{200\_80\_4}  & \textbf{715}{[}\textbf{+0},+6{]}          & \textbf{715}{[}\textbf{+0},+18{]}        & 1911    \\
\multicolumn{1}{l|}{250\_100\_4} & \textbf{1305}{[}\textbf{+0},\textbf{+0}{]}         & \textbf{1305}{[}\textbf{+0},+2{]}        & 2200    \\
\multicolumn{1}{l|}{300\_120\_4} & \textbf{1257}{[}+30,+1315{]}     & \textbf{1257}{[}+33,+1315{]}    & 2572    \\
\multicolumn{1}{l|}{350\_140\_4} & \textbf{1737}{[}+1426,+1426{]}   & \textbf{1737}{[}+332.5,+1426{]} & 3163    \\
\multicolumn{1}{l|}{400\_160\_4} & N/A{[}N/A,N/A{]}        & \textbf{2240}{[}+638,+1296{]}   & 3536    \\
\multicolumn{1}{l|}{450\_180\_4} & N/A{[}N/A,N/A{]}        & \textbf{1869}{[}+2172,+2172{]}  & N/A     \\
\multicolumn{1}{l|}{500\_200\_4} & N/A{[}N/A,N/A{]}        & \textbf{2577}{[}+1686,+2036{]}  & N/A     \\
\multicolumn{1}{l|}{100\_40\_6}  & \textbf{140}{[}\textbf{+0},+2{]}          & \textbf{140}{[}\textbf{+0},+6{]}         & \textbf{140*}    \\
\multicolumn{1}{l|}{150\_60\_6}  & \textbf{402}{[}\textbf{+0},+1{]}          & \textbf{402}{[}\textbf{+0},+1{]}         & 787     \\
\multicolumn{1}{l|}{200\_80\_6}  & \textbf{477}{[}\textbf{+0},+6{]}          & \textbf{477}{[}\textbf{+0},+13{]}        & 1274    \\
\multicolumn{1}{l|}{250\_100\_6} & \textbf{870}{[}\textbf{+0},+4{]}          & \textbf{870}{[}\textbf{+0},+9{]}         & 1467    \\
\multicolumn{1}{l|}{300\_120\_6} & \textbf{839}{[}\textbf{+0},+21{]}         & \textbf{839}{[}\textbf{+0},+278{]}       & 1715    \\
\multicolumn{1}{l|}{350\_140\_6} & \textbf{1158}{[}+6,+951{]}       & \textbf{1158}{[}+120,+951{]}    & 2109    \\
\multicolumn{1}{l|}{400\_160\_6} & \textbf{1493}{[}+417,+864{]}     & \textbf{1493}{[}\textbf{+0},+864{]}      & 2357    \\
\multicolumn{1}{l|}{450\_180\_6} & 1829{[}+865,N/A{]}      & \textbf{1246}{[}+824.5,+1448{]} & N/A     \\
\multicolumn{1}{l|}{500\_200\_6} & N/A{[}N/A,N/A{]}        & \textbf{1718}{[}+1238,+1357{]}  & N/A     \\ \hline
\end{tabular}
\end{table}

\begin{table}[]
\caption{Results on WSNO under 3600s time limit}
\label{tab:wsno3600}

\begin{tabular}{llll}
\hline
\multicolumn{1}{l|}{Instances}   & Deci-LSPBO          & LS-PBO              & PBO-IHS     \\ \hline
\multicolumn{1}{l|}{n m k}       & min{[}median,max{]} & min{[}median,max{]} &      \\ \hline                              
\multicolumn{1}{l|}{100\_40\_4}  & \textbf{210}{[}\textbf{+0},\textbf{+0}{]}      & \textbf{210}{[}\textbf{+0},\textbf{+0}{]}      & \textbf{210*} \\
\multicolumn{1}{l|}{150\_60\_4}  & \textbf{602}{[}\textbf{+0},\textbf{+0}{]}      & \textbf{602}{[}\textbf{+0},\textbf{+0}{]}      & 1180 \\
\multicolumn{1}{l|}{200\_80\_4}  & \textbf{715}{[}\textbf{+0},+3{]}      & \textbf{715}{[}\textbf{+0},+3{]}      & 1911 \\
\multicolumn{1}{l|}{250\_100\_4} & \textbf{1305}{[}\textbf{+0},\textbf{+0}{]}     & \textbf{1305}{[}\textbf{+0},\textbf{+0}{]}     & 2200 \\
\multicolumn{1}{l|}{300\_120\_4} & \textbf{1257}{[}\textbf{+0},+13{]}    & \textbf{1257}{[}\textbf{+0},+17{]}    & 2572 \\
\multicolumn{1}{l|}{350\_140\_4} & \textbf{1737}{[}\textbf{+0},+46{]}    & \textbf{1737}{[}\textbf{+0},+72{]}    & 3163 \\
\multicolumn{1}{l|}{400\_160\_4} & \textbf{2240}{[}\textbf{+0},\textbf{+0}{]}     & \textbf{2240}{[}\textbf{+0},+10{]}    & 3536 \\
\multicolumn{1}{l|}{450\_180\_4} & \textbf{1869}{[}\textbf{+0},+281{]}   & \textbf{1869}{[}\textbf{+0},+381{]}   & 4041 \\
\multicolumn{1}{l|}{500\_200\_4} & \textbf{2577}{[}+47,+2036{]} & \textbf{2577}{[}+6,+2036{]}  & 4613 \\
\multicolumn{1}{l|}{100\_40\_6}  & \textbf{140}{[}\textbf{+0},\textbf{+0}{]}      & \textbf{140}{[}\textbf{+0},\textbf{+0}{]}      & \textbf{140*} \\
\multicolumn{1}{l|}{150\_60\_6}  & \textbf{402}{[}\textbf{+0},\textbf{+0}{]}      & \textbf{402}{[}\textbf{+0},\textbf{+0}{]}      & 787  \\
\multicolumn{1}{l|}{200\_80\_6}  & \textbf{477}{[}\textbf{+0},+4{]}      & \textbf{477}{[}\textbf{+0},+7{]}      & 1274 \\
\multicolumn{1}{l|}{250\_100\_6} & \textbf{870}{[}\textbf{+0},+5{]}      & \textbf{870}{[}\textbf{+0},+12{]}     & 1467 \\
\multicolumn{1}{l|}{300\_120\_6} & \textbf{839}{[}\textbf{+0},+9{]}      & \textbf{839}{[}\textbf{+0},+9{]}      & 1715 \\
\multicolumn{1}{l|}{350\_140\_6} & \textbf{1158}{[}\textbf{+0},+4{]}     & \textbf{1158}{[}\textbf{+0},+53{]}    & 2109 \\
\multicolumn{1}{l|}{400\_160\_6} & \textbf{1493}{[}\textbf{+0},+9{]}     & \textbf{1493}{[}\textbf{+0},+20{]}    & 2357 \\
\multicolumn{1}{l|}{450\_180\_6} & \textbf{1246}{[}\textbf{+0},+12{]}    & \textbf{1246}{[}\textbf{+0},+79{]}    & 2694 \\
\multicolumn{1}{l|}{500\_200\_6} & \textbf{1718}{[}\textbf{+0},+15{]}    & \textbf{1718}{[}\textbf{+0},+501{]}   & 3075 \\ \hline
\end{tabular}  
\end{table}

\subsection{Experimental results on SAP}

\begin{table}[]
\caption{Results on SAP under 300s time limit}
\label{tab:sap300}

\begin{tabular}{llll}
\hline
\multicolumn{1}{l|}{Instances} & Deci-LSPBO               & LS-PBO              & PBO-IHS \\ \hline
\multicolumn{1}{l|}{n}         & Min{[}median,max{]}      & Min{[}median,max{]} &         \\ \hline
\multicolumn{1}{l|}{100}       & \textbf{579}{[}+7,+10{]} & 583{[}+4,+8{]}      & N/A     \\
\multicolumn{1}{l|}{110}       & \textbf{622}{[}+8,+10{]}          & 623{[}+6.5,+11{]}   & N/A     \\
\multicolumn{1}{l|}{120}       & \textbf{679}{[}+9,+13{]}          & 683{[}+7,+10{]}     & N/A     \\
\multicolumn{1}{l|}{130}       & \textbf{743}{[}+5,+12{]}          & 744{[}+5,+10{]}     & N/A     \\
\multicolumn{1}{l|}{140}       & \textbf{758}{[}+11,+15{]}         & 764{[}+5,+10{]}     & N/A     \\
\multicolumn{1}{l|}{150}       & \textbf{823}{[}+10,+14{]}         & 825{[}+7,+12{]}     & N/A     \\
\multicolumn{1}{l|}{160}       & \textbf{862}{[}+15,+21{]}         & 868{[}+11,+15{]}    & N/A     \\
\multicolumn{1}{l|}{170}       & \textbf{901}{[}+11,+16{]}         & 902{[}+10,+17{]}    & N/A     \\
\multicolumn{1}{l|}{180}       & \textbf{972}{[}+11,+16{]}         & 974{[}+9,+15{]}     & N/A     \\
\multicolumn{1}{l|}{190}       & \textbf{1002}{[}+10,+17{]}        & 1005{[}+7.5,+15{]}  & N/A     \\
\multicolumn{1}{l|}{200}       & \textbf{1065}{[}+15,+22{]}        & 1067{[}+13,+21{]}   & N/A     \\
\multicolumn{1}{l|}{210}       & 1112{[}+6,+18{]}         & \textbf{1111}{[}+9,+17{]}    & N/A     \\
\multicolumn{1}{l|}{220}       & \textbf{1148}{[}+23,+30{]}        & 1162{[}+9,+18{]}    & N/A     \\
\multicolumn{1}{l|}{230}       & \textbf{1197}{[}+15,+23{]}        & 1199{[}+13.5,+20{]} & N/A     \\
\multicolumn{1}{l|}{240}       & 1231{[}+10,+18{]}        & \textbf{1230}{[}+12,+19{]}   & N/A     \\
\multicolumn{1}{l|}{250}       & \textbf{1286}{[}+11,+19{]}        & \textbf{1286}{[}+13.5,+21{]} & N/A     \\
\multicolumn{1}{l|}{260}       & \textbf{1321}{[}+20,+30{]}        & 1329{[}+14,+22{]}   & N/A     \\
\multicolumn{1}{l|}{270}       & \textbf{1389}{[}+21.5,+33{]}      & 1400{[}+11,+19{]}   & N/A     \\
\multicolumn{1}{l|}{280}       & \textbf{1416}{[}+16,+29{]}        & 1417{[}+16,+25{]}   & N/A     \\
\multicolumn{1}{l|}{290}       & 1467{[}+11,+22{]}        & \textbf{1457}{[}+21.5,+31{]} & N/A     \\
\multicolumn{1}{l|}{300}       & 1542{[}+14,+22{]}        & \textbf{1533}{[}+22.5,+32{]} & N/A     \\ \hline
\end{tabular}
\end{table}

\begin{table}[]

\caption{Results on SAP under 3600s time limit}
\label{tab:sap3600}

\begin{tabular}{llll}
\hline
\multicolumn{1}{l|}{Instances} & Deci-LSPBO          & LS-PBO              & PBO-IHS \\ \hline
\multicolumn{1}{l|}{n}         & Min{[}median,max{]} & Min{[}median,max{]} &         \\ \hline
\multicolumn{1}{l|}{100}       & \textbf{579}{[}+5,+8{]}     & 582{[}+2,+4{]}      & N/A     \\
\multicolumn{1}{l|}{110}       & \textbf{620}{[}+6,+9{]}      & 623{[}+3,+6{]}      & N/A     \\
\multicolumn{1}{l|}{120}       & \textbf{678}{[}+6,+9{]}      & 680{[}+5,+8{]}      & N/A     \\
\multicolumn{1}{l|}{130}       & \textbf{737}{[}+6,+11{]}     & 739{[}+5,+10{]}     & N/A     \\
\multicolumn{1}{l|}{140}       & \textbf{752}{[}+11,+14{]}    & 758{[}+6,+9{]}      & N/A     \\
\multicolumn{1}{l|}{150}       & \textbf{818}{[}+9,+12{]}     & 821{[}+6.5,+11{]}   & N/A     \\
\multicolumn{1}{l|}{160}       & 862{[}+10.5,+15{]}  & \textbf{861}{[}+12,+17{]}    & N/A     \\
\multicolumn{1}{l|}{170}       & \textbf{896}{[}+11+14{]}     & 901{[}+5+10{]}      & N/A     \\
\multicolumn{1}{l|}{180}       & \textbf{964}{[}+13,+17{]}    & 970{[}+6.5,+13{]}   & N/A     \\
\multicolumn{1}{l|}{190}       & \textbf{995}{[}+10,+17{]}    & 998{[}+7,+17{]}     & N/A     \\
\multicolumn{1}{l|}{200}       & \textbf{1065}{[}+9.5,+13{]}  & 1066{[}+9,+13{]}    & N/A     \\
\multicolumn{1}{l|}{210}       & \textbf{1095}{[}+18,+25{]}   & 1104{[}+8,+14{]}    & N/A     \\
\multicolumn{1}{l|}{220}       & 1148{[}+14,+23{]}   & \textbf{1143}{[}+21.5,+27{]} & N/A     \\
\multicolumn{1}{l|}{230}       & \textbf{1192}{[}+12,+19{]}   & \textbf{1192}{[}+13,+18{]}   & N/A     \\
\multicolumn{1}{l|}{240}       & \textbf{1219}{[}+15,+21{]}   & 1227{[}+7,+14{]}    & N/A     \\
\multicolumn{1}{l|}{250}       & 1273{[}+15,+22{]}   & \textbf{1272}{[}+15.5+24{]}  & N/A     \\
\multicolumn{1}{l|}{260}       & \textbf{1304}{[}+29,+38{]}   & 1322{[}+12.5,+20{]} & N/A     \\
\multicolumn{1}{l|}{270}       & 1389{[}+13,+20{]}   & \textbf{1387}{[}+16,+23{]}   & N/A     \\
\multicolumn{1}{l|}{280}       & 1411{[}+12,+21{]}   & \textbf{1409}{[}+17,+20{]}   & N/A     \\
\multicolumn{1}{l|}{290}       & \textbf{1453}{[}+17,+24{]}   & 1457{[}+14,+21{]}   & N/A     \\
\multicolumn{1}{l|}{300}       & \textbf{1529}{[}+20,+29{]}   & 1533{[}+16,+24{]}   & N/A     \\ \hline
\end{tabular}
\end{table}

Experimental results on SAP instances are shown in Table \ref{tab:sap300} and Table \ref{tab:sap3600}. For each instance, the best results among the three algorithms are marked in bold font. PBO-IHS is unable to find any feasible solution for any instance on both 300s and 3600s time limit. Therefore, we focus on the comparison between LS-PBO and our DeciLS-PBO.

For 300s and 3600s time limit, the best results among 30 runs of DeciLS-PBO surpass that of LS-PBO on most of the instances, which indicates our DeciLS-PBO is more effective in exploring the search space and has a better search ability at both early and later stages.

\subsection{Experimental results on PB Competition 2016}

We tested 1600 instances from the most recent PB Competition \cite{pb16}. Given an instance, an algorithm wins on the instance means that its best result among all the runs is better than that of any other solver. For the PB16 benchmark, We count the win numbers of each algorithm and report the result in Table \ref{tab:pb16}. 

\begin{table}[]
\caption{Results on PB Competition 2016 (value: \#win)}
\label{tab:pb16}
\begin{tabular}{lllll}
\hline
\multicolumn{1}{l}{TimeLimit}  &\#inst  & DeciLS-PBO &  LS-PBO & PBO-IHS \\
\hline            
300s &1600  & 1130 & 1045 &\textbf{1171}               \\
\hline
3600s &1600  &\textbf{1215}    & 1142                  & 1185      \\
\hline
\end{tabular}
\end{table}

For 300s time limit, it can be observed in Table \ref{tab:pb16} that the performance of the complete solver PBO-IHS surpasses that of incomplete ones, and DeciLS-PBO achieves better performance than LS-PBO. For 3600s time limit, DeciLS-PBO outperforms both LS-PBO and the complete solver PBO-IHS.

\subsection{Empirical analysis on DeciLS-PBO}

To give a more detailed analysis of our DeciLS-PBO algorithm, we conduct further empirical studies to test the effectiveness of each component separately. Since DeciLS-PBO introduced two new components on the basis of the existing solver LS-PBO, we transform it into two alternative solvers:
\begin{itemize}
\item DeciLS-PBO$_{alt1}$: LS-PBO with Care-FC scheme.
\item DeciLS-PBO$_{alt2}$: LS-PBO with UP-based decimation as the initial assignment generator.
\end{itemize}

We compare both alternative solvers with LS-PBO and DeciLS-PBO. The performance of DeciLS-PBO$_{alt1}$ can reflect the effectiveness of our Care-FC scheme, whereas that of DeciLS-PBO$_{alt2}$ can reflect the effectiveness of the UP-based decimation algorithm. We conduct the experiment on the same instances as mentioned above, and set the same time limit which is 300s and 3600s. \#inst denotes the total number of instances the corresponding problem has. We compare the best results among 300 runs of each algorithm on each instance. For a given instance, an algorithm wins on this instance means that its best result among 30 runs is better than that of any other solver. We count the win numbers of each algorithm on each type of benchmark and arrange the results in Table \ref{tab:last1} and Table \ref{tab:last2}. The results in bold font indicate the best performance for the related problem.

From the experimental results in Table \ref{tab:last1} and \ref{tab:last2},  we have the following observations:
\begin{enumerate}[(1)]
    \item DeciLS-PBO outperforms its two alternative versions on all benchmarks except WSNO under 300s time limit, which indicates the remarkable performance of DeciLS-PBO, nevertheless, a drawback of UP-based initialization also exposes. For WSNO results under 300s time limit, the algorithms with UP-based initialization i.e., DeciLS-PBO and DeciLS-PBO$_{alt2}$ win 13 among 18 instances, whereas the algorithms without UP-based initialization i.e., LS-PBO and DeciLS-PBO$_{alt1}$ win 18 among 18 instances. Therefore, combining UP-based decimation for initialization brings a bad time consumption on extremely large instances, consequently, the algorithm is unable to find a feasible solution in a short time. For the 3600s time limit, algorithms with UP-based initialization on WSNO reach the same level as others, which indicates the UP-based decimation is still promising in longer term.
    \item Compared to LS-PBO on instances except WSNO under 300s time limit, the results of DeciLS-PBO$_{alt1}$ and DeciLS-PBO$_{alt2}$ demonstrate that both Care-FC scheme and UP-based decimation improve the LS-PBO separately. Moreover, it reflects the effectiveness of the two components at the initial stage. For the 3600s time limit, the performance of LS-PBO on the PB16 benchmark surpasses that of DeciLS-PBO$_{alt1}$ and DeciLS-PBO$_{alt2}$, but the gap does not vary a lot.

\end{enumerate}

\begin{table}[]
\caption{Results of DeciLS-PBO and its alternative versions on all instances (value: \#win)}
\label{tab:last1}

\begin{tabular}{llllll}
\hline
\multicolumn{1}{l}{Benchmark}  & \#inst & DeciLS-PBO & DeciLS-PBO$_{alt1}$ & DeciLS-PBO$_{alt2}$\\
\hline
\multicolumn{5}{l}{TimeLimit=300s}   \\                               
MWCB & 24     & \textbf{17}   & 5    & 2    \\
WSNO & 18     & 13   & \textbf{18}  & 13     \\
SAP  & 21     & \textbf{12}     & 7    & 8     \\
PB16 &1600 & \textbf{1168} &1101 &1100\\
 \hline 
 total &1663 &\textbf{1210} &1131 &1123\\
\hline
\multicolumn{5}{l}{TimeLimit=3600s}    \\ 

MWCB & 24     & \textbf{14}  & 6   & 4  \\
WSNO & 18 & \textbf{18} & \textbf{18}  & \textbf{18} \\
SAP  & 21     & \textbf{16}   & 5   & 2  \\
PB16 &1600 &\textbf{1249} &1203 &1202\\
\hline
total &1663 &\textbf{1297} &1232 &1226\\
\hline
\end{tabular}

\end{table}

\begin{table}[]
\caption{Results of LS-PBO and alternative versions of DeciLS-PBO on all instances (value: \#win)}
\label{tab:last2}

\begin{tabular}{llllll}
\hline
\multicolumn{1}{l}{Benchmark}  & \#inst & LS-PBO & DeciLS-PBO$_{alt1}$ & DeciLS-PBO$_{alt2}$\\
\hline
\multicolumn{5}{l}{TimeLimit=300s}   \\                              
MWCB & 24    & 4 & \textbf{13}   & 7        \\
WSNO & 18     & \textbf{18}  & \textbf{18}  & 13     \\
SAP  & 21     & 7  & 9 & \textbf{11}     \\
PB16 &1600 &1100 &1120 &\textbf{1142}\\
\hline
total & 1663 &1129 & 1160 &\textbf{1173}\\
\hline
\multicolumn{5}{l}{TimeLimit=3600s}    \\ 

MWCB & 24     & 3 & \textbf{11}  & 10  \\
WSNO & 18 & \textbf{18} & \textbf{18}  & \textbf{18} \\
SAP  & 21     & 7 & \textbf{13}   & 10 \\
PB16 &1600 &\textbf{1221} &1205 &1210\\
\hline
total & 1663 &\textbf{1249}  &1247 &1248 \\
\hline
\end{tabular}

\end{table}

\section{Conclusions}
\label{sec:con}
Pseudo Boolean Optimization is an important constraint optimization problem with many real-world applications. In this paper, we propose a new local search-based PBO solver dubbed DeciLS-PBO. Based on the existing solver LS-PBO, we introduced two components: an UP-based decimation algorithm in the initial stage to generate a better initial assignment, and a new scheme called Care-FC based on heuristic emphasizing clauses (constraints) called care to help choose a more suitable clause (constraint) when the search gets stuck. Experimental results on MWCB, WSNO, SAP, and PB Competition 2016 show that DeciLS-PBO outperforms its competitors on most of the instances above, which indicates the superiority of our new solver. In future work, we would like to further study the effectiveness of our Care-FC scheme on other solvers. The local search-based solvers share a similar framework, as a result, our Care-FC scheme can be directly applied to the existing local search-based PMS solvers.

\subsubsection*{Acknowledgements}
The authors declare that there is no conflict of interest regarding the publication of this paper. This work was partially supported by the National Natural Science Foundation of China (Grant Nos. 62076108 and 61872159), and the education department of Jilin Province (JJKH20211106KJ, JJKH20211103KJ).

\bibliographystyle{alpha}
\bibliography{main}

\end{document}